\documentclass[journal]{IEEEtran}
\usepackage{amssymb}
\usepackage{amsmath}
\usepackage{dcolumn}
\usepackage{bm}
\usepackage{color}
\usepackage{graphicx}
\usepackage{subfigure} 
\usepackage{tabularx}
\usepackage{array}
\usepackage{amsthm}
\usepackage{cite}

\newtheorem{theorem}{Theorem}
\newtheorem{corollary}{Corollary}
\newtheorem{lemma}{Lemma}
\newtheorem{remark}{Remark}
\newtheorem{definition}{Definition}
\newtheorem{assumption}{Assumption}
\newtheorem{proposition}{Proposition}

\renewcommand{\t}{^{\mbox{\tiny\sf T}}}
\newcommand{\bremark}{\begin{remark}
\begin{rm}}
\newcommand{\eremark}{ \end{rm}\hfill \rule{1mm}{2mm}
\end{remark} }
\newcommand{\btheorem}{\begin{theorem} \begin{it}}
\newcommand{\etheorem}{\end{it} \hfill \rule{1mm}{2mm}
\end{theorem} }
\newcommand{\blemma}{\begin{lemma} \begin{it} }
\newcommand{\elemma}{ \end{it} \hfill\rule{1mm}{2mm}
\end{lemma} }
\newcommand{\bcorollary}{\begin{corollary} \begin{it} }
\newcommand{\ecorollary}{ \end{it} \hfill\rule{1mm}{2mm}
\end{corollary} }
\newcommand{\bdefinition}{\begin{definition} }
\newcommand{\edefinition}{ \hfill\rule{1mm}{2mm}
\end{definition} }
\newcommand{\bproposition}{\begin{proposition} }
\newcommand{\eproposition}{\hfill \rule{1mm}{2mm}
\end{proposition} }
\newcommand{\bexample}{\begin{example} \begin{rm}}
\newcommand{\eexample}{ \end{rm} \hfill\rule{1mm}{2mm}
\end{example} }

\newcommand{\bassumption}{\begin{assumption} }
\newcommand{\eassumption}{\hfill \rule{1mm}{2mm}
\end{assumption} }

\newcommand{\balgorithm}{\medskip\begin{algorithm} \rm}
\newcommand{\ealgorithm}{ \hfill \rule{1mm}{2mm}\medskip
\end{algorithm} }

\newcommand{\basm}{\begin{assumption} \begin{rm} }
\newcommand{\easm}{ \end{rm} \hfill\rule{1mm}{2mm}
\end{assumption} }

\begin{document}

\title{Coordinated Guiding Vector Field Design for Ordering-Flexible Multi-Robot Surface Navigation }

\author{
\IEEEauthorblockN{Bin-Bin Hu,Hai-Tao Zhang\IEEEauthorrefmark{1}, 
Weijia Yao, 
Zhiyong Sun, 
 and Ming Cao\IEEEauthorrefmark{1}
 }
\thanks{
 Bin-Bin Hu and Hai-Tao Zhang are with the School of Artificial Intelligence and Automation, the Engineering Research Center of Autonomous Intelligent Unmanned Systems, the Key Laboratory of Image Processing and Intelligent Control, and the State Key Lab of Digital Manufacturing Equipment and Technology, Huazhong University of Science and Technology, Wuhan 430074, P.R.~China (e-mails: binbin.hu@ntu.edu.sg, zht@mail.hust.edu.cn).
}
\thanks{Weijia Yao is with the School of Robotics, Hunan University, Hunan 410082, P.R.~China. He is also is with the Institute of Engineering and Technology, University
of Groningen, 9747 AG Groningen, the Netherlands (e-mail:
w.yao@rug.nl). }
\thanks{Zhiyong Sun is with
Department of Electrical Engineering, Eindhoven University of Technology,
the Netherlands (e-mail: z.sun@tue.nl).}
\thanks{Ming Cao is with the Institute of Engineering and Technology,
University of Groningen, 9747 AG Groningen, the Netherlands (e-mail: m.cao@rug.nl).
}
}

\maketitle

\begin{abstract}
We design a distributed coordinated guiding vector field (CGVF) for a group of robots to achieve ordering-flexible 
motion coordination while maneuvering on a desired two-dimensional (2D) surface. The CGVF is characterized by three terms, i.e., 
a convergence term to drive the robots to converge to the desired surface, a propagation term to provide a traversing direction for maneuvering on the desired surface, 
and a coordinated term to achieve the surface motion coordination with an arbitrary ordering of the robotic group. By setting the surface parameters
as additional virtual coordinates, the proposed approach eliminates potential singularity of the CGVF and enables both the global convergence to the desired surface and the maneuvering on the surface from all possible initial conditions. The ordering-flexible surface motion coordination is realized by each robot to share with its neighbors only two virtual coordinates, i.e. that of a given target and that of its own, which reduces the communication and computation cost in multi-robot {\it surface navigation}.
Finally, the effectiveness of the CGVF is substantiated by
extensive numerical simulations.

\end{abstract}
 

\IEEEpeerreviewmaketitle

\section{Introduction}
Multi-robot navigation is an important research topic in robotics due to its wide applications such as area coverage.
In specific navigation operations, such as surveillance,  convoy and rescue, patrolling, 
and defense applications \cite{macwan2014multirobot,hu2020multiple}, one common requirement for a group of robots is to accurately follow
possibly different paths and coordinate their motions subject to some geometric or parametric constraints. Such missions are generally referred to as multi-robot path navigation, or coordinated path following \cite{peng2005coordinating}.
Significantly, one of the most fundamental techniques in multi-robot path navigation is the path-following algorithm, which has been 
widely studied in recent years. Researchers have utilized projection points \cite{aguiar2007trajectory}, line of sight (LOS)\cite{rysdyk2006unmanned} or guiding vector fields \cite{kapitanyuk2017guiding}. Among these path-following algorithms, the guiding-vector-field (GVF) approach achieves a higher path-following accuracy with a lower control energy cost, which guarantees the global convergence for some self-intersecting curves \cite{yao2021singularity}.

In this context, researchers have considered a more challenging multi-robot path navigation problem with the assistance of projection-point and LOS methods in the literature. For instance, early efforts were devoted to straight-line paths with adaptive controllers \cite{burger2009straight}. For multi-robot path navigation following circular paths, a nested-invariant-set approach was proposed in \cite{doosthoseini2015coordinated}. Other relevant methods concerning circular paths were investigated in \cite{hu2021distributed1,yao2019distributed,hu2021bearing}. For increasingly sophisticated missions, more attention is paid to more general parameterized curves~\cite{ghommam2010formation,ghabcheloo2009coordinated}.  
Afterwards, for non-parameterized paths, a distributed hybrid control law  was developed in~\cite{lan2011synthesis}. Moreover, simple closed and  periodic-changing paths were investigated in~\cite{sabattini2015implementation}.

Another research line in multi-robot path navigation is the GVF algorithm. Till now, there are only a few studies using vector fields. Recently, for the tracking of desired circular paths with multiple unmanned aerial vehicles, a distributed vector-field algorithm was proposed in \cite{de2017circular}. Afterwards, a vector-field law integrating potential functions was designed in \cite{nakai2013vector} to enable the global convergence to geometrical patterns such as lines and circles, but the relative distances among neighboring robots cannot be tuned. These two studies \cite{de2017circular,nakai2013vector} only considered simple curves in a 2D plane. More recently, a decentralized control law in~\cite{pimenta2013decentralized} has extended the results to the 3D Euclidean space. By utilizing a singularity-free GVF with a virtual coordinate in \cite{yao2021singularity}, a novel navigation method was developed in \cite{yao2021multi} to guide multiple robots to cooperatively follow possibly distinct desired paths in an $n$-dimensional Euclidean space. Later, it was extended to a spontaneous-ordering platoon whereas maneuvering along a predefined path~\cite{hu2023spontaneous}.

However, in some navigation tasks such as pipeline maintenance and curved-area patrolling, robots are required to accurately follow desired surfaces rather than desired paths while coordinating their motions. We call such tasks the multi-robot {\it surface navigation} tasks, which have potential applications in terrain hazard detection, seafloor mapping, etc. Compared with desired paths, a desired surface involves more sophisticated topological characteristics, such as the singularity in GVF, which inevitably gives rise to more challenges. 
Moreover, as the surface often features non-zero curvature, the inter-robot distance is not the common Euclidean distance but the geodesic distance, which is arduous to calculate and hence the performance in coordinating the robot's motion may degrade.

Another challenging issue arising from multi-robot navigation is ordering-flexible coordination, which is crucial in dynamical environments. 
Most of the conventional multi-robot path navigation algorithms \cite{burger2009straight,doosthoseini2015coordinated,ghommam2010formation,hu2021bearing,hu2021distributed1,yao2019distributed,ghabcheloo2009coordinated,sabattini2015implementation,de2017circular,nakai2013vector,pimenta2013decentralized,yao2021multi} only considered the coordination with fixed orderings, which predefines and fixes each agent's neighbor relation with a specific ordering. Recently, the notes~\cite{montenbruck2017fekete,yao2022guidingArxiv} studied the multi-robot formation coordination which automatically
arranges space distances among each pair of robots on some manifolds, but the orderings are still fixed.
By contrast, ordering-flexible coordination is to coordinate all robots' motions with an arbitrary ordering and thus no ordering is predetermined for specific robots, enhancing the flexibility of coordination in a confined space, e.g., narrow passage. However, \cite{lan2011synthesis,hu2023spontaneous} only studied such ordering-flexible coordination in the path navigation scenario. Leveraging such ordering-flexible coordination on the {\it surface navigation} task still remains a challenging problem.

In this note,  we extend the multi-robot path navigation to multi-robot {\it surface navigation} and address the aforementioned three challenging issues,  i.e., i) singularity during the process of surface convergence and maneuvering, ii) costly calculation of inter-robot coordination and iii) ordering-flexible motion coordination. 
Specifically, by setting the surface parameters as two additional virtual coordinates, we extend the dimensions of the coordinated guiding vector field (CGVF) to eliminate singularity. Meanwhile, by incorporating neighboring robots' virtual coordinates and the target's virtual coordinates, we achieve ordering-flexible multi-robot {\it surface navigation}. 
 In summary, the main contributions of this note are three-fold as follows.

\begin{enumerate}

\item  Distinct from the previous path and surface navigation algorithms \cite{burger2009straight,doosthoseini2015coordinated,ghommam2010formation,hu2021bearing,hu2021distributed1,yao2019distributed,ghabcheloo2009coordinated,sabattini2015implementation,de2017circular,nakai2013vector,pimenta2013decentralized,yao2021multi,montenbruck2017fekete,yao2022guidingArxiv}, which focus on the ordering-fixed coordination, we design a distributed CGVF for multiple robots to achieve more flexible coordinated motions with arbitrary orderings while maneuvering on a desired 2D surface.

\item Compared with the ordering-flexible navigation~\cite{lan2011synthesis} only featuring local convergence to the desired path, the proposed approach guarantees the global convergence to ordering-flexible {\it surface navigation} and eliminates potential singularity of the GVF by setting the surface parameters as two additional virtual coordinates.

\item Compared with the ordering-fixed surface coordination \cite{montenbruck2017fekete} via arduous calculation of inter-robot geodesic distance, the proposed approach can reduce both the communication and computation costs of multi-robot ordering-flexible {\it surface navigation} by the sensing of only two virtual coordinates.

\end{enumerate}

The rest of the note is organized below. The preliminaries and problem are formulated in Section~II.  The main technical result is derived in Section III. The effectiveness of the proposed CGVF is verified in Section IV. Finally, the conclusion is drawn in Section V.

{\it Notations:} The real numbers and positive real numbers are denoted by $\mathbb{R},\mathbb{R}^+$, respectively. The $n$-dimensional Euclidean space is denoted by $\mathbb{R}^n$. The integer numbers are denoted by $\mathbb{Z}$. The notation $\mathbb{Z}_i^j$ represents the set $\{m\in \mathbb{Z}~|~i\leq m\leq j\}$.  
The Euclidean norm of a vector $v$ is $\|v\|$. The $n$-dimensional identity matrix is represented by $I_n$.

\section{Preliminaries}
Consider a multi-robot system consisting of $N$ robots represented by ${\cal V}=\{1,2,\dots, N\}$. Each robot is described by the following kinematic equation, 
\begin{align}
\label{kinetic_F}
 \dot{x}_i &=u_i, i\in\mathcal V,
\end{align}
where $x_i(t) :=[x_{i,1}, \dots, x_{i,n}]\t\in\mathbb{R}^n$, $u_i(t):=[u_{i,1}, \dots, u_{i,n}]\t\in\mathbb{R}^n$ denote the positions and the control inputs of robot $i$ in the $n$-dimensional Euclidean space, respectively.
The definitions concerning CGVF and ordering-flexible {\it surface navigation}, assumptions, and problem formulation are introduced below.

\subsection{Higher-Dimensional GVF for Surface }
Suppose a desired 2D surface $\mathcal S^{phy}$ in the $n$-dimensional Euclidean space is described by a zero-level set
of $n-2$ implicit functions $\phi_i$,
\begin{align}
\label{implicit_path_definition}
\mathcal S^{phy}:=\{ \sigma\in\mathbb{R}^n~|~\phi_i(\sigma)=0, i=1,\dots, n-2\},
\end{align}
where $\sigma\in\mathbb{R}^n$ are the coordinates and $\phi_i(\cdot): \mathbb{R}^n\rightarrow\mathbb{R}$ are twice continuously differentiable, i.e., $\phi(\cdot)\in C^2$.
Using $\mathcal S^{phy}$ in \eqref{implicit_path_definition},  we are ready to introduce the GVF $\chi^{phy}\in\mathbb{R}^n$ (see, e.g., \cite{kapitanyuk2017guiding}),
\begin{align}
\label{eq_GVF}
\chi^{phy}=\times \big(\nabla\phi_1(\sigma), \cdots, \nabla\phi_{n-2}(\sigma), \mathbf{m}\big)-\sum_{i=1}^{n-2}k_{i}\phi_i(\sigma)\nabla\phi_i(\sigma),
\end{align}
where $k_i\in\mathbb{R}^+$ is the control gain, $\times(\cdot)$ represents the cross product, $\nabla\phi_i(\cdot): \mathbb{R}^{n}\rightarrow\mathbb{R}^{n}$ denotes the gradient of $\phi_i$ w.r.t. $\sigma$, and $\mathbf{m}:=[m_1, m_2, \dots, m_n]\t\in\mathbb{R}^n$ is an auxiliary column vector with at least one non-zero element. The GVF in \eqref{eq_GVF} consists of two terms: the first propagation term $\times (\nabla\phi_1, \cdots, \nabla\phi_{n-2}, \mathbf{m})$ is orthogonal to all the gradients $\nabla\phi_i, i\in \mathbb{Z}_1^{n-2}$ whereas the auxiliary vector $\mathbf{m}$ provides a traversing direction for maneuvering on the desired surface; the second convergence term $\sum_{i=1}^{n-2}k_{i}\phi_i\nabla\phi_i$ is to guide robots to approach the desired surface $\mathcal S^{phy}$. 
However, the GVF in \eqref{eq_GVF} may fail to work if there exist singular points (i.e., $\chi^{phy}=\mathbf{0}$), which hinders the global convergence. To eliminate the undesirable singular points and achieve the global convergence, we define a higher-dimensional GVF for the corresponding surface $\mathcal S^{phy}$ below.

\begin{definition}
\label{definition_GVF}
(Higher-dimensional GVF) \cite{yao2021singularity} Given a desired surface $\mathcal S^{phy}$ parameterized by 
$\mathcal S^{phy}:=\{[\sigma_1, \cdots, \sigma_n]\t\in\mathbb{R}^n~|~\sigma_j=f_j(\omega_1, \omega_2 ), j\in\mathbb{Z}_1^n\}$ 
with the $j$-th cooridinate $\sigma_j\in\mathbb{R}$, the virtual coordinate parameters $\omega_1, \omega_2\in\mathbb{R}$, and the twice continuously differentiable function $ f_j \in C^2$. We define the corresponding higher-dimensional surface 
$\mathcal S^{hgh}:=\{\xi\in\mathbb{R}^{n+2}~|~ \phi_j(\xi)=0, j\in\mathbb{Z}_1^n\}$ with the generalized coordinate vector $\xi:=[\sigma_1,\dots, \sigma_n, \omega_1, \omega_2]\t$ and $\phi_{j}(\xi)=\sigma_j-f_j(\omega_1, \omega_2)\in\mathbb{R}$. Since $\mathcal S^{phy}$ corresponds to the projection of $\mathcal S^{hgh}$ on the first $n$ coordinates,
the higher-dimensional GVF $\chi^{hgh}\in\mathbb{R}^{n+2}$ is defined to be
\begin{align}
\label{high_eq_GVF1}
\chi^{hgh}=\times \big(\nabla\phi_1(\xi), \cdots, \nabla\phi_n(\xi), \mathbf{m}\big)-\sum_{j=1}^nk_{j}\phi_j(\xi)\nabla\phi_j(\xi),
\end{align}
where $\mathbf{m}:=[m_1, m_2, \dots, m_{n+1}, m_{n+2}]\t\in\mathbb{R}^{n+2}$ denotes the auxiliary column vector in Eq.~\eqref{eq_GVF}, $k_j\in\mathbb{R}^+$ is the control gain of surface function~$\phi_j$, and $\nabla\phi_j:=[0,\dots,1,\dots,-\partial{f}_{j}^{[1]}, -\partial{f}_{j}^{[2]} ]\t\in\mathbb{R}^{n+2}, j\in\mathbb{Z}_1^n$ represents the gradient of $\phi_j(\xi)$ w.r.t. $\xi$. Here, $\partial{f}_{j}^{[1]}:={\partial f_{j}(\omega_{1}, \omega_{2})}/{\partial\omega_{1}}$ and $\partial{f}_{j}^{[2]}:={\partial f_{j}(\omega_{1}, \omega_{2})}/{\partial\omega_{2}}$ denote the partial derivatives of $f_{j}(\omega_{1}, \omega_{2})$ w.r.t. $\omega_{1}$ and $\omega_{2}$, respectively. 
To simplify the calculation of $\times (\nabla\phi_1, \cdots, \nabla\phi_n, \mathbf{m})$, we set $\mathbf{m}=[0, \dots, 0, m_{n+1}, m_{n+2}]\t\in\mathbb{R}^{n+2}$ with any non-zero constants $m_{n+1}\neq0, m_{n+2}\neq0$, and then calculate the higher-dimensional GVF $\chi^{hgh}\in\mathbb{R}^{n+2}$ to be
\begin{align}
\label{high_eq_GVF2}
      \chi^{hgh}  =\begin{bmatrix}
        	              (-1)^n (m_{n+2}\partial{f}_{1}^{[1]}-m_{n+1}\partial{f}_{1}^{[2]})-k_{1}\phi_{1} \\
	              \vdots\\
                      (-1)^n(m_{n+2}\partial{f}_{n}^{[1]}-m_{n+1}\partial{f}_{n}^{[2]})-k_{n}\phi_{n}\\
                      (-1)^nm_{n+2}+\sum\limits_{j=1}^nk_{j}\phi_{j}\partial{f}_{j}^{[1]}\\
                      -(-1)^nm_{n+1}+\sum\limits_{j=1}^nk_{j}\phi_{j}\partial{f}_{j}^{[2]}\\
        \end{bmatrix},   
\end{align}
which recovers the desired surface $\mathcal S^{phy}$ by projecting $\chi^{hgh}$ to the first $n$-dimensional Euclidean space.
\end{definition}

In Definition~\ref{definition_GVF}, the additional coordinates $\omega_{1}, \omega_2$ ensure the singularity-free property of $\chi^{hgh}$ (i.e., $\chi^{hgh}\neq0$) \cite{yao2021singularity} due to the non-zero terms $(-1)^nm_{n+2}, -(-1)^nm_{n+1}$ in Eq.~\eqref{high_eq_GVF2}, which thus eliminates the singular points of the original GVF $\chi^{phy}$ in~\eqref{eq_GVF} and in turn guarantees well-defined vector fields for all states.

\subsection{Ordering-Flexible Multi-Robot {\it Surface Navigation}}
From Definition~\ref{definition_GVF}, we are ready to consider the task of multi-robot {\it surface navigation}.
Suppose that the $i$-th desired surface $\mathcal S_i^{phy}$ for the $i$-th robot $i\in \mathcal V$ in the $n$-dimensional Euclidean space is described as, 
\begin{align}\label{desired_surface}
\mathcal S_i^{phy}=\{\sigma_i\in\mathbb{R}^n~|~\phi_{i,j}(\sigma_i)=0, j\in\mathbb{Z}_1^n\},
\end{align}
where $\sigma_i:=[\sigma_{i,1} , \dots, \sigma_{i,n}]\t$ are the coordinates of $\mathcal S_i^{phy}$, 
and $\phi_{i,j}(\sigma_i):=\sigma_{i,j}-f_{i,j}(\omega_{i,1}, \omega_{i,2})$ are the zero-level implicit functions of the $i$-th desired surface. Here, 
$f_{i,j}(\omega_{i,1}, \omega_{i,2})\in C^2,  j\in\mathbb{Z}_1^n$ are the parametric functions, and $\omega_{i,1}, \omega_{i,2}$ are the virtual coordinates of the $i$-th desired surface.
Following Definition~\ref{definition_GVF}, the $i$-th desired surface $\mathcal S_i^{phy}$ is transformed to the corresponding higher-dimensional surface $\mathcal S_i^{hgh}=\{\xi_i\in\mathbb{R}^{n+2}~|~\phi_{i,j}(\xi_i)=0, j\in\mathbb{Z}_1^n\}$
with the generalized coordinates $\xi_i:=[\sigma_{i,1} , \dots, \sigma_{i,n}, \omega_{i,1},$ $ \omega_{i,2}]\t$.

Defining $p_i:=[x_{i,1}, \dots, x_{i,n}, \omega_{i,1}, \omega_{i,2}]\t\in\mathbb{R}^{n+2}$ and substituting the position $x_i=[x_{i,1}, \dots, x_{i,n}]\t$ of the $i$-th robot in \eqref{kinetic_F} into the $i$-th desired surface $\mathcal S_i^{hgh}$,  the surface-convergence error $\phi_{i,j}(p_i), \forall j\in\mathbb{Z}_1^n,$ between robot $i$ and the desired higher-dimensional surface $\mathcal S_i^{hgh}$ becomes 
\begin{align}
\label{err_phi}
\phi_{i,j}(p_i)=&x_{i,j}-f_{i,j}(\omega_{i,1}, \omega_{i,2}),  j\in \mathbb{Z}_1^n.
\end{align}
\begin{figure}[!htb]
  \centering
  \includegraphics[width=7cm]{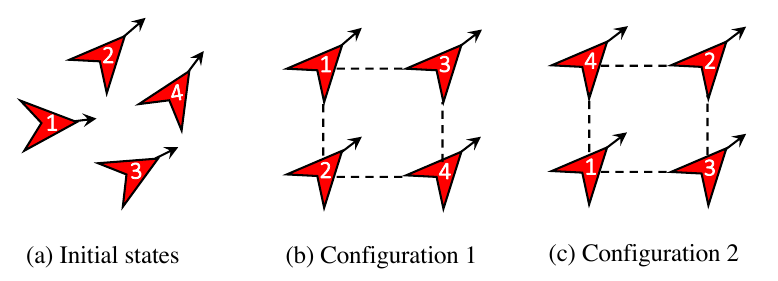}
  \caption{Illustration of the ordering-flexible motion coordination. }
  \label{ordering_free}
\end{figure}
Moreover, let $\Omega_i:=[\omega_{i,1}, \omega_{i,2}]\t\in\mathbb{R}^2$, and one has that the sensing neighborhood $\mathcal N_i(t)$ of robot $i$ is
\begin{align}
\label{sening_neighbor}
\mathcal N_i(t):=\{k\in {{\cal V}},k\neq i \;\big|~ \| \Omega_{i,k}(t)\| <R\}
\end{align} 
with $\Omega_{i,k}:=\Omega_i-\Omega_k$, the sensing range $R>r\in\mathbb{R}^+$ and the specified safe range $r\in\mathbb{R}^+$ between any two robots.

\begin{remark}
Different from calculating the complicated geodesic distance to determine the neighborhood $\mathcal N_i(t)$, we hereby consider a more convenient method for $\mathcal N_i(t)$ using the relative value of the virtual coordinates $\|\Omega_{i,k}\|$ in Eq.~\eqref{sening_neighbor} rather than the complicated geodesic distances, which reduces sensing and calculation burdens. To facilitate the understanding of $\mathcal N_i(t)$ in \eqref{sening_neighbor}, an intuitive example is given below. For a desired horizontal plane surface in \eqref{desired_surface}, e.g., $\phi_{i,1}=\sigma_{i,1}-\omega_{i,1}, \phi_{i,2}=\sigma_{i,2}-\omega_{i,2}, i\in \mathcal{V}$, the relative value $\Omega_{i,k}$ becomes the conventional Euclidean distance in a 2D plane, which implies that $\mathcal N_i(t)$ in \eqref{sening_neighbor} is reasonable.  
\end{remark}

Moreover, as $\mathcal N_i(t)$ in \eqref{sening_neighbor} is calculated according to the time-varying $\|\Omega_{i,k}(t)\|$, the neighborhood set
$\mathcal N_i(t)$  is time-varying as well. This leads to an ordering-flexible 
coordination whereas poses some new challenging issues in the stability analysis.
As shown in Figs.~\ref{ordering_free} (b) and (c), the configurations~$1$ and $2$ both satisfy the same pattern of the four robots but with different orderings. 

To achieve the ordering-flexible coordination on the surface, we also introduce a virtual target robot $ v^\ast$ moving on the desired surface $\mathcal S_i^{phy}$ with the designed $\chi^{hgh}$ in Eq.~\eqref{high_eq_GVF2}, which is to provide an attraction term for other robots. Define  $\Omega^{\ast}:=[\omega^{\ast}_1, \omega^{\ast}_2]\t\in\mathbb{R}^2$. As the virtual robot $ v^\ast$ is already on the desired surface, it holds that $\phi_{\ast,j}=0, \forall j\in\mathbb{Z}_1^n$, which implies that 
the derivatives of $\Omega^{\ast}=[\omega_1^{\ast}, \omega_2^{\ast}]\t$ become the last two terms of Eq. \eqref{high_eq_GVF2}, namely, $\dot{\omega}_1^{\ast}=(-1)^nm_{n+2}+\sum_{j=1}^nk_{\ast,j}\phi_{\ast,j}\partial{f}_{\ast,j}^{[1]}=(-1)^nm_{n+2}, \dot{\omega}_2^{\ast}=-(-1)^nm_{n+1}+\sum_{j=1}^nk_{\ast,j}\phi_{\ast,j}\partial{f}_{\ast,j}^{[2]}=-(-1)^nm_{n+1}.$
Rewriting $\dot{\omega}_1^{\ast}, \dot{\omega}_2^{\ast}$ in a compact form as 
 \begin{align}
 \label{dynamic_virtual_target}
\dot{\Omega}^{\ast}=[\dot{\omega}^{\ast}_1, \dot{\omega}^{\ast}_2]\t=[(-1)^nm_{n+2}, -(-1)^nm_{n+1}]\t. 
 \end{align} 

\begin{definition}
\label{CPF_definition}
(Ordering-flexible multi-robot {\it surface navigation}) 
The group of robots ${\cal V}$ collectively forms an ordering-flexible coordinated navigation maneuvering on the desired surface $\mathcal {S}^{phy}$ if the following objectives are accomplished,

1. (\textbf{Surface convergence})~All the robots converge to the desired surface, i.e., $\lim_{t\rightarrow\infty} \phi_{i,j}(p_i(t))=0,\forall i\in\mathcal V, j\in \mathbb{Z}_1^n$.

2. (\textbf{Surface maneuvering})~All the robots maneuver on the desired surface and maintain a desirable  motion coordination pattern, i.e.,
$\lim_{t\rightarrow\infty} \dot{\Omega}_{i}(t)=\lim_{t\rightarrow\infty} \dot{\Omega}_{k}(t)\neq\mathbf{0}, \forall i\neq k \in{\cal V}$ with $\dot{\Omega}_i$ denoting the time derivative of $\Omega_i$ of robot $i$.

3. (\textbf{Ordering-flexible coordination}) All the robots coordinate their ordering-flexible motions distributively using neighboring robots’ virtual coordinates and the target virtual coordinates, i.e.,    
\begin{align}
\label{definition_condi_ordering_free}
&(a) \lim_{t\rightarrow\infty} \big(\frac{1}{N}\sum_{i=1}^N\Omega_{i}(t)-\Omega^{\ast}(t)\big)=\mathbf{0}, \forall i \in{\cal V}, \nonumber\\
&(b)~r<\lim_{t\rightarrow\infty} \|\Omega_{i}(t)-\Omega_{k}(t)\|<R,  \forall i\in{\cal V}, k\in \mathcal N_i(t),
\end{align} 
where $R\in\mathbb{R}^+, r\in\mathbb{R}^+$ are the specified sensing and safe radius in \eqref{sening_neighbor}, respectively.
\end{definition}

\subsection{Preliminaries on Assumptions}
\label{assumption_description}

\begin{assumption}
\label{assp_error} \cite{yao2021singularity} For any given $\kappa>0$, an arbitrary point $p_0\in\mathbb{R}^n$ and a desirable physical surface $\mathcal{S}_i^{phy}$, we assume that
$$\inf\{\|\phi(p_0)\|: \mbox{dist}(p_0,\mathcal S^{phy})\geq \kappa\}>0.$$
\end{assumption}

\begin{assumption}
\label{assp_intial_value}
 The initial virtual coordinates of the robots satisfy
$\|\Omega_{i,k}(0)\|>r,  \; \forall i\neq k \in {\cal V}$ with $r$ given in \eqref{definition_condi_ordering_free}.
\end{assumption}

Assumption~\ref{assp_intial_value} is necessary for Condition~(b): $\lim_{t\rightarrow\infty}$ $\|\Omega_{i,k}\|>r$ in Definition~\ref{CPF_definition}.

\begin{assumption}
\label{assp_area}
The total area $\mathcal A_i^{phy}$ of the desired surface $\mathcal{S}_i^{phy}$ is assumed to be sufficiently bigger than the required area when forming multi-robot {\it surface navigation}.
\end{assumption}

Assumption~\ref{assp_area} guarantees that there exists a sufficiently large area to distribute or contain all robots $\mathcal V$ freely.

\begin{assumption}
\label{assp_derivative}
The first and second derivatives of $f_{i,j}(\omega_{i,1}, $ $\omega_{i,2}), i\in\mathcal V, j\in \mathbb{Z}_1^n$ w.r.t. $\omega_{i,1}$ and $\omega_{i,2}$  are all bounded, i.e., 
\begin{align*}
&\bigg\|\frac{\partial f_{i,j}(\omega_{i,1}, \omega_{i,2})}{\partial\omega_{i,l}}\bigg\|\leq\varsigma_{i,1}, i\in\mathcal V, j\in \mathbb{Z}_1^n, l\in\mathbb{Z}_1^2,\nonumber\\
&\bigg\|\frac{\partial f_{i,j}(\omega_{i,1}, \omega_{i,2})}{\partial\omega_{i,l}\partial\omega_{i,q}}\bigg\|\leq\varsigma_{i,2}, i\in\mathcal V, j\in \mathbb{Z}_1^n, l, q\in\mathbb{Z}_1^2
\end{align*}
for some unknown positive constants $\varsigma_{i,1}, \varsigma_{i,2}\in \mathbb{R}^+$.
\end{assumption}
Assumption~\ref{assp_derivative} prevents the parameterization of the surface from changing too fast, which is a necessary condition for the convergence analysis in Lemma~\ref{lemma_step_2} later.

\subsection{Problem Formulation} 
Let $\partial{f}_{i,j}^{[1]}, \partial{f}_{i,j}^{[2]}$ be the partial derivatives of $f_{i,j}(\omega_{i,1}, \omega_{i,2})$ w.r.t. $\omega_{i,1}, \omega_{i,2}$, namely,
\begin{align}
\label{gradient_f_omega12}
\partial{f}_{i,j}^{[1]}:=&\frac{\partial f_{i,j}(\omega_{i,1}, \omega_{i,2})}{\partial\omega_{i,1}},~\partial{f}_{i,j}^{[2]}:=&\frac{\partial f_{i,j}(\omega_{i,1}, \omega_{i,2})}{\partial\omega_{i,2}}.
\end{align}
It follows from Eqs.~\eqref{err_phi} and~\eqref{gradient_f_omega12} that the gradient $\nabla\phi_{i,j}(p_i)$ of the surface function $\phi_{i,j}(p_i)$ along the vector $p_i=[x_{i,1}, \dots, x_{i,n}, \omega_{i,1}, \omega_{i,2}]\t$ can be calculated by $\nabla\phi_{i,j}(p_i):=[0, \dots, 1, \dots, -\partial{f}_{i,j}^{[1]},-\partial{f}_{i,j}^{[2]}]\t\in\mathbb{R}^{n+2}$,
where the $j$-th component of the gradient vector $\nabla\phi_{i,j}(p_i)$ is~$1$.
Moreover, one has $\dot{\phi}_{i,j}(p_i)=\nabla\phi_{i,j}(p_i)\t\dot{p}_i, i\in\mathcal V, j\in \mathbb{Z}_1^n.$
Let $\Phi_i(p_i)=[\phi_{i,1}, \phi_{i,2}, \dots, \phi_{i,n}]\t$, $u_i=[u_{i,1}, u_{i,2}, \dots, u_{i,n}]\t$, and $u_i^\Omega:=[\omega_{i,1}^r, \omega_{i,2}^r]\t$ be the desired inputs of the virtual coordinates $\Omega_i$.
Combining the dynamics of the virtual coordinates $\dot{\Omega}_i:=u_i^\Omega$ with the $i$-th robot dynamics in \eqref{kinetic_F},  one has that
\begin{align}
\label{dynamic_path}
\begin{bmatrix}
\dot{\Phi}_i\\
\dot{\Omega}_{i}\\
\end{bmatrix}
=
D_i 
\begin{bmatrix}
u_i\\
u_i^\Omega\\
\end{bmatrix}
\end{align}
with $\Phi_i=\Phi_i(p_i), \phi_{i,j}= \phi_{i,j}(p_i)$ and 
\begin{align*}
D_i=& \begin{bmatrix}                      
                       1 & 0 & \cdots & 0 & -\partial{f}_{i,1}^{[1]} & -\partial{f}_{i,1}^{[2]}\\
                       0 & 1 &\cdots & 0 & -\partial{f}_{i,2}^{[1]} & -\partial{f}_{i,2}^{[2]}\\
                       \vdots & \vdots & \ddots & \vdots &\vdots & \vdots\\
                       0 & 0& \cdots &1 & -\partial{f}_{i,n}^{[1]} & -\partial{f}_{i,n}^{[2]}\\
                       0 & 0 &\cdots &0 & 1 & 0\\
                       0 & 0 &\cdots &0 & 0 & 1\\
                       \end{bmatrix}\in \mathbb{R}^{(n+2)\times (n+2)}.
\end{align*}

Now, we are ready to introduce the main problem addressed in this note.

{\bf Problem 1}: (Ordering-flexible multi-robot {\it surface navigation})
Design a distributed control signal 
\begin{align}
\label{pro_desired_signal}
\{u_{i}, \omega_{i,1}^r, \omega_{i,2}^r \}:=&\chi_i^{hgh}(\phi_{i,j}, \partial{f}_{i,j}^{[1]}, \partial{f}_{i,j}^{[2]}, \Omega_{i}, \Omega_{k}, \Omega^{\ast}), \nonumber\\
                                                                  &\forall i\in{\cal V}, j\in\mathbb{Z}_1^n, k\in \mathcal N_i(t),
\end{align}
as a CGVF for the multi-robot system composed of \eqref{kinetic_F} and \eqref{pro_desired_signal} to attain surface convergence and maneuvering, and ordering-flexible coordination (i.e., Objectives 1-3 in Definition~\ref{CPF_definition}).

\section{Main Technical Result }
Let the auxiliary column vector $\mathbf{m}$ in \eqref{high_eq_GVF1} be $\mathbf{m}=[0,\dots,0,-1,1]\t\in\mathbb{R}^{n+2}$ with $m_{n+1}=-1, m_{n+2}=1$ for simplicity. It then follows from $\chi^{hgh}$ in \eqref{high_eq_GVF2} that 
the CGVF $\chi_i^{hgh}$ for robot $i$ can be designed as,
\begin{align}
\label{desired_law}
u_{i,j}=&(-1)^n( \partial{f}_{i,j}^{[1]}+ \partial{f}_{i,j}^{[2]})-k_{i,j}\phi_{i,j}, j\in \mathbb{Z}_1^n, \nonumber\\
\omega_{i,1}^r=&(-1)^n+\sum\limits_{j=1}^nk_{i,j}\phi_{i,j} \partial{f}_{i,j}^{[1]}-c_i(\omega_{i,1}-\widehat{\omega}_{i,1})+\eta_{i,1},\nonumber\\
\omega_{i,2}^r=&(-1)^n+\sum\limits_{j=1}^nk_{i,j}\phi_{i,j} \partial{f}_{i,j}^{[2]}-c_i(\omega_{i,2}-\widehat{\omega}_{i,2})+\eta_{i,2},
\end{align}
where $k_{i,j}, c_i\in\mathbb{R}^+, i\in\mathcal V, j\in \mathbb{Z}_1^n$ are the control gains of surface-convegence errors and virtual coordinates, respectively, $\phi_{i,j}$, $\partial{f}_{i,j}^{[1]}, \partial{f}_{i,j}^{[2]}, j\in\mathbb{Z}_1^n$ are given in \eqref{err_phi} and \eqref{gradient_f_omega12} for conciseness. The two virtual coordinates for robot $i$ are denoted by $\omega_{i,1}, \omega_{i,2}$, respectively. $\widehat{\Omega}_i:=[\widehat{\omega}_{i,1}, \widehat{\omega}_{i,2}]\t$ denotes the estimates of the target virtual coordinates $\Omega^{\ast}=[\omega^{\ast}_1, \omega^{\ast}_2]\t$.
The inter-robot repulsive terms $\eta_{i,1}, \eta_{i,2}$ are given by 
\begin{align}
\label{de_eta}
\eta_{i,1}=&\sum_{k\in \mathcal N_i(t)}\alpha(\|\Omega_{i,k}\|)\frac{\omega_{i,1}-\omega_{k,1}}{\|\Omega_{i,k}\|},\nonumber\\
\eta_{i,2}=&\sum_{k\in \mathcal N_i(t)}\alpha(\|\Omega_{i,k}\|)\frac{\omega_{i,2}-\omega_{k,2}}{\|\Omega_{i,k}\|},
\end{align}
where $\Omega_{i,k}:=\Omega_i-\Omega_k$ and $\mathcal N_i(t)$ are given in Eq.~\eqref{sening_neighbor}, and $\alpha(s)$ denotes the continuous repulsive function between any two virtual coordinates $\Omega_{i}, \Omega_{k}$. 
Since $\|\Omega_{i,k}(0)\|>r,  \; \forall i\neq k \in {\cal V}$ in Assumption~\ref{assp_intial_value}, the function $\alpha(s):(r, \infty)\rightarrow [0, \infty)$ (see e.g.~\cite{chen2019cooperative})
is designed to satisfy
\begin{align}
\label{alpha}
\alpha(s)=0, \forall s\in[R, \infty),
\lim_{s\rightarrow r^{+}}\alpha(s)=\infty
\end{align}
with the sensing range $R$, the safe range $r$ given in~\eqref{sening_neighbor}, and the right limit $r^{+}$.  
An example of $\alpha(s)$ satisfying~\eqref{alpha} is given by  (see, e.g.~\cite{chen2019cooperative}),
\begin{equation}
\label{exam_poten_func}
\alpha(s)=
\left\{
\begin{array}{llr}
\frac{(s-R)^2}{(s-r)^2}& r<s\leq R,\\
0 & s>R.
\end{array}
\right.
\end{equation}
Here, $\alpha(s)$ is monotonicly decreasing when $s\in(r, R]$ and equals $0$ when $s\in(R, \infty)$, which is locally continuous in the domain $(r, \infty)$.  Then, $\dot{\alpha}(s)$ becomes, 

\begin{equation}
\label{derivative_poten_func}
\dot{\alpha}(s)=
\left\{
\begin{array}{llr}
\frac{2(s-R)(R-r)}{(s-r)^3}& r<s\leq R,\\
0 & s>R,
\end{array}
\right.
\end{equation}
and one has that $\lim_{s\rightarrow R^{-}}\dot{\alpha}(s)=\lim_{s\rightarrow R^{+}}\dot{\alpha}(s)=0$ with $R^{-}, R^{+}$ being the left and right limits of $R$, respectively, which implies that $\dot{\alpha}(s)$ is continuous over the domain $s\in(r, \infty)$ and thus different from the one in \cite{chen2019cooperative}.

\begin{remark}
\label{remark_c3}
When the target virtual coordinates $\Omega^{\ast}(t)$ are only available to some robots, $\widehat{\Omega}_i, i\in \mathcal V,$ in Eq.~\eqref{desired_law} are designed to be the distributed estimators for $\Omega^{\ast}$ with additional communication networks. Since $\Omega^{\ast}(t)$ changes at a constant velocity in Eq.~\eqref{dynamic_virtual_target}, such a problem has been nicely solved by the following two typical approaches in many MAS studies \cite{zhao2013distributed,hong2006tracking} that cannot be overlooked.
 The first one is to broadcast $\Omega^{\ast}(t)$ to every robot via a communication network to achieve finite-time convergence (see, e.g., \cite{zhao2013distributed}), i.e., $\lim_{t\rightarrow T}\big(\widehat{\Omega}_i(t)-\Omega^{\ast}(t)\big)=\mathbf{0}, i\in \mathcal V,$ with a constant $T>0$. The other is to estimate $\Omega^{\ast}$ under exponential convergence, i.e., $\lim_{t\rightarrow\infty}\big(\widehat{\Omega}_i(t)-\Omega^{\ast}(t)\big)=\mathbf{0}, i\in \mathcal V,$
exponentially, see, e.g., \cite{hong2006tracking}. These two approaches are beyond the main scope of this note, but the vanishing estimation errors may lead to a finite escape time phenomenon even for an initially stable nonlinear system~\cite{freeman1995global}. 
\end{remark}

Since the inter-robot repulsive terms $\eta_{i,1}, \eta_{i,2}$ in Eqs.~\eqref{de_eta} contain $\|\Omega_{i,k}\|$ at the denominator, the closed-loop system governed by the CGVF $\chi_i^{hgh}$ in \eqref{desired_law} may not be well defined if $\|\Omega_{i,k}(t)\|=r$ or $\|\Omega_{i,k}(t)\|=0, \forall i\in\mathcal V, k\in\mathcal N_i$ during the navigation process. Accordingly, we divide the main technical results into three steps. 
In Step~1, we prove the uniqueness of solution of the closed-loop system with Eq.~\eqref{desired_law} (i.e., $\|\Omega_{i,k}(t)\|\neq r, \mbox{or}, \|\Omega_{i,k}(t)\|\neq 0, \forall i\in\mathcal V, k\in\mathcal N_i, t\geq0$). 
In Step~2, we prove surface convergence and maneuvering (Objectives 1-2 in Definition~\ref{CPF_definition}). In Step 3, we give the ordering-flexible coordination (Objective~3 in Definition~\ref{CPF_definition}).

\subsection{ The Uniqueness of Solution of Closed-Loop System}
 
Let $\widetilde{\Omega}_i:=[\widetilde{\omega}_{i,1}, \widetilde{\omega}_{i,2}]\t$ be the coordinate errors between the $i$-th virtual coordinates $\Omega_i$ and the target virtual coordinates $\Omega^{\ast}$, where $\widetilde{\omega}_{i,1}:= \omega_{i,1}-\omega_1^{\ast},~\widetilde{\omega}_{i,2}:=\omega_{i,1}-\omega_2^{\ast}.$
Let $F_i^{[1]}:=[\partial{f}_{i,1}^{[1]}, \dots, \partial{f}_{i,n}^{[1]}]\t\in \mathbb{R}^{n}, F_i^{[2]}:=[\partial{f}_{i,1}^{[2]}, \dots, $ $ \partial{f}_{i,n}^{[2]}]\t\in \mathbb{R}^{n}$, $K_i:=\mbox{diag}\{k_{i,1}, $ $\dots, k_{i,n}\}\in \mathbb{R}^{n\times n}$ and $I_n\in \mathbb{R}^{n\times n}$ be the identity matrix. Then, substituting Eq.~\eqref{desired_law} into Eq.~\eqref{dynamic_path} and the derivative of $\widetilde{\omega}_{i,1}, \widetilde{\omega}_{i,2}$ yields the closed-loop system
\begin{align}
\label{dynamic_path_4}
\begin{bmatrix}
   \dot{\Phi}_i\\
   \dot{\widetilde{\omega}}_{i,1}\\
   \dot{\widetilde{\omega}}_{i,2}\\
\end{bmatrix} =& \begin{bmatrix}                      
                       -K_i\Big(I_n+F_i^{[1]}(F_i^{[1]})\t+F_i^{[2]}(F_i^{[2]})\t\Big)\Phi_i \\
                       \Phi_i\t K_iF_i^{[1]}\\
                        \Phi_i\t K_iF_i^{[2]}
                       \end{bmatrix}\nonumber\\
                      & +
                        \begin{bmatrix}   
                         -F_i^{[1]}(\epsilon_{i,1}+e_{i,1})-F_i^{[2]}(\epsilon_{i,2}+e_{i,2})   \\
                       \epsilon_{i,1}+e_{i,1}\\
                       \epsilon_{i,2}+e_{i,2}\\
                       \end{bmatrix},                                
\end{align}
where $\Phi_i$ is given in \eqref{dynamic_path} and 
\begin{align}
\label{replace_coordination}
\epsilon_{i,1}:=&-c_i\widetilde{\omega}_{i,1}+\eta_{i,1},~\epsilon_{i,2}:=-c_i\widetilde{\omega}_{i,2}+\eta_{i,2},\nonumber\\
e_i:=&
\begin{bmatrix}
e_{i,1}\\
 e_{i,2}
 \end{bmatrix}=c_i(\widehat{\Omega}_i-\Omega^{\ast}).
\end{align}
Here, $e_i$ denotes the exponentially vanishing estimation error from the estimation of the target virtual coordinates. Recalling Remark \ref{remark_c3} where $\lim_{t\rightarrow\infty}\widehat{\Omega}_i(t)-\Omega^{\ast}(t)=\mathbf{0}$ exponentially, one has  
\begin{align}
\label{convergence_e}
\lim_{t\rightarrow\infty}e_{i,1}(t)=0, \lim_{t\rightarrow\infty}e_{i,2}(t)=0,
\end{align}
exponentially. The uniqueness of solution of the closed-loop system~\eqref{dynamic_path_4} is guaranteed if the virtual coordinates $\|\Omega_{i,k}\|, i\in\mathcal V, k\in\mathcal N_i$ always stay in
\begin{align}
\label{unique_space}
\mathbb{S} =\{\|\Omega_{i,k}(t)\|\in(r, \infty)\}, \forall t\geq0,
\end{align} 
i.e., $\|\Omega_{i,k}(t)\|>r, \forall t\geq0 \Rightarrow \|\Omega_{i,k}(t)\|\neq r, \mbox{or}, \|\Omega_{i,k}(t)\|$ $\neq 0, \forall t\geq0.$

Next, we give the lemma on the uniqueness of solution.

\begin{lemma}
\label{lemma_bounded_P}
Under Assumptions~\ref{assp_error}-\ref{assp_area}, a multi-robot system governed by \eqref{kinetic_F}, \eqref{desired_law} and \eqref{de_eta} guarantees the uniqueness of solution for $ \forall t\geq 0$, i.e., $\|\Omega_{i,k}(t)\|\neq r, \mbox{or}, \|\Omega_{i,k}(t)\|\neq 0, \forall i\in\mathcal V, k\in\mathcal N_i, t\geq0$. 
\end{lemma}

\begin{proof}

Under Assumption~\ref{assp_intial_value}, the initial values of the virtual coordinates satisfy $\|\Omega_{i,k}(0)\|>r, i\in\mathcal V, k\in\mathcal N_i$
for $t=0$. 
We prove by contradiction. 
If the claim in \eqref{unique_space} does not hold for $\forall t>0$, there exists a
finite time $T>0$, such that the states $\|\Omega_{i,k}(t)\|$ stay in $\mathbb{S}$ for $t\in [0,T)$ but not $t=T$, which implies that there exists at least one pair inter-robot virtual coordinates $\Omega_{i,k}$ satisfying 
\begin{align}
\label{err_claim}
\|\Omega_{i,k}(T)\|\in[0, r].
\end{align}
We choose the candidate Lyapunov function 
\begin{align}
\label{V_1}
V(t)=&\sum\limits_{i\in\mathcal{ V}}\bigg\{\frac{\Phi_i\t K_i\Phi_i}{2}+\frac{c_i\widetilde{\Omega}_i\t\widetilde{\Omega}_i}{2}\bigg\}+\sum\limits_{i\in\mathcal{V}}\sum\limits_{k\in\mathcal N_i(t)}\int_{\|\Omega_{i,k}\|}^{R}\alpha(s)ds,
\end{align} 
which is nonnegative and differentiable due to the definition of $\alpha(s)$ in~\eqref{exam_poten_func}. Moreover, it follows from Eqs.~\eqref{de_eta} and \eqref{exam_poten_func} that $V(t)$ approaches $\infty$ if $\|\Omega_{i, k}(t)\|=r$ or $\|\Omega_{i, k}(t)\|=0$.

For $t\in[0, T)$, the states $\Omega_{i,k}$ satisfy $\|\Omega_{i,k}(t)\|\in(r, \infty), \forall i\in\mathcal V, k\in\mathcal N_i$, one has that the closed-loop system~\eqref{dynamic_path_4} is well defined and has the uniqueness of solution on the time interval $[0,T)$, which further implies that $V(t)$ in~\eqref{V_1} is well defined
for $t\in[0, T)$. Then, the partial derivatives of $V(t), t\in[0, T)$ w.r.t. $\Phi_{i}, \Omega^{\ast}, \Omega_i, $ are ${\partial V(t)}/{\partial \Phi_{i}\t}=\Phi_{i}\t K_{i}, {\partial V(t)}/{\partial (\Omega^{\ast})\t}=-c_i\widetilde{\Omega}_i\t, {\partial V(t)}/{\partial \Omega_i\t}=c_i\widetilde{\Omega}_i\t-\sum_{k\in\mathcal N_i(t)}\alpha(\|\Omega_{i,k}\|){\Omega_{i,k}\t}/{\|\Omega_{i,k}\|},$
which implies that the derivative of $V(t), t\in[0, T)$ along the system's trajectory is
\begin{align}
\label{dot_V}
\frac{dV}{dt}
		  =&\sum_{i\in{\cal V}}\bigg\{\Phi_{i}\t K_{i}\bigg(-K_i\Big(I_n+F_i^{[1]}(F_i^{[1]})\t+F_i^{[2]}(F_i^{[2]})\t\Big)\Phi_i  \nonumber\\
		    &-F_i^{[1]}(\epsilon_{i,1}+e_{i,1})-F_i^{[2]}(\epsilon_{i,2}+e_{i,2})\bigg)+\bigg(c_i\widetilde{\Omega}_i\t\nonumber\\
		    &-\sum\limits_{k\in\mathcal N_i(t)}\alpha(\|\Omega_{i,k}\|)\frac{\Omega_{i,k}\t}{\|\Omega_{i,k}\|}\bigg)\dot{\Omega}_i-c_i\widetilde{\Omega}_i\t\dot{\Omega}^{\ast}\bigg\}.
\end{align}
Recalling the fact $\dot{\widetilde{\Omega}}_i=\dot{\Omega}_i-\dot{\Omega}^{\ast}$, one has $c_i\widetilde{\Omega}_i\t\dot{\Omega}_i-c_i\widetilde{\Omega}_i\t\dot{\Omega}^{\ast}=c_i\widetilde{\Omega}_i\t\dot{\widetilde{\Omega}}_i$.
Meanwhile, it follows from the definition of $\alpha(s)$ in \eqref{exam_poten_func} that $\sum_{i\in \mathcal V}\sum_{k\in\mathcal N_i(t)}\alpha(\|\Omega_{i,k}\|)\frac{\Omega_{i,k}\t}{\|\Omega_{i,k}\|}\dot{\Omega}^{\ast}=\dot{\Omega}^{\ast}\sum_{i\in \mathcal V}\sum_{k\in\mathcal N_i(t)}\alpha(\|\Omega_{i,k}\|)\frac{\Omega_{i,k}\t}{\|\Omega_{i,k}\|}=\mathbf{0},$ which immediately leads to $
\sum_{i\in \mathcal V}$ $\sum_{k\in\mathcal N_i(t)}\alpha(\|\Omega_{i,k}\|)\frac{\Omega_{i,k}\t}{\|\Omega_{i,k}\|}\dot{\Omega}_i
=\sum_{i\in{\cal V}}$ $\sum_{k\in\mathcal N_i(t)}$ $\alpha(\|\Omega_{i,k}\|)\frac{\Omega_{i,k}\t}{\|\Omega_{i,k}\|}\dot{\widetilde{\Omega}}_i.
$
Then, one has that
$\sum_{i\in{\cal V}}\{(c_i\widetilde{\Omega}_i\t-\sum_{k\in\mathcal N_i(t)}\alpha(\|\Omega_{i,k}\|)\frac{\Omega_{i,k}\t}{\|\Omega_{i,k}\|})\dot{\Omega}_i-c_i\widetilde{\Omega}_i\t\dot{\Omega}^{\ast}\}
=\sum_{i\in{\cal V}}(c_i\widetilde{\Omega}_i\t-\sum_{k\in\mathcal N_i(t)}\alpha(\|\Omega_{i,k}\|)\frac{\Omega_{i,k}\t}{\|\Omega_{i,k}\|})\dot{\widetilde{\Omega}}_i.$

From the definitions of $\Omega_{i,k}, \widetilde{\Omega}_i$ in Eqs.~\eqref{de_eta} and~\eqref{dynamic_path_4}, it follows from Eq.~\eqref{replace_coordination} that $\sum_{i\in{\cal V}}(c_i\widetilde{\Omega}_i\t-\sum_{k\in\mathcal N_i(t)}\alpha(\|\Omega_{i,k}\|)\frac{\Omega_{i,k}\t}{\|\Omega_{i,k}\|})\dot{\widetilde{\Omega}}_i=-\epsilon_{i,1}\dot{\widetilde{\omega}}_{i,1}-\epsilon_{i,2}\dot{\widetilde{\omega}}_{i,2}$.
Combining with $\dot{\widetilde{\omega}}_{i,1}, \dot{\widetilde{\omega}}_{i,2}$ in \eqref{dynamic_path_4}, it follows from Eq.~\eqref{dot_V} that
\begin{align}
\label{dot_V1}
\frac{dV}{dt}
                     =&\sum_{i\in{\cal V}}\bigg\{-\Phi_{i}\t K_{i}\t K_i\Phi_i-\beta_{i,1}-\beta_{i,2}\bigg\}
\end{align}
with  
$
\beta_{i,1}:=-\Phi_{i}\t K_{i}\t F_i^{[1]}(F_i^{[1]})\t K_i\Phi_i-\epsilon_{i,1}^2-\Phi_{i}\t K_{i}F_i^{[1]}e_{i,1}-2\Phi_{i}\t K_{i}F_i^{[1]}\epsilon_{i,1}-\epsilon_{i,1}e_{i,1}$, and
$\beta_{i,2}:=-\Phi_{i}\t K_{i}\t F_i^{[2]}(F_i^{[2]})\t K_i\Phi_i$ $-\epsilon_{i,2}^2-\Phi_{i}\t K_{i}F_i^{[2]}e_{i,2}-2\Phi_{i}\t K_{i}F_i^{[2]}\epsilon_{i,2}-\epsilon_{i,2}e_{i,2}.
$

Using the definition of $F_i^{[1]}, \Phi_i, K_i$ in \eqref{dynamic_path_4}, one has that $(F_i^{[1]})\t K_i\Phi_i=\Phi_{i}\t K_{i}\t F_i^{[1]}$ is a scalar, which implies that
$\beta_{i,1}=-(\Phi_{i}\t K_iF_i^{[1]}+\epsilon_{i,1}+{e_{i,1}}/{2})^2+{e_{i,1}^2}/{4}.$

Analogously, using the fact $(F_i^{[2]})\t K_i\Phi_i=\Phi_{i}\t K_{i}\t F_i^{[2]}$, one has $\beta_{i,2}=-(\Phi_{i}\t K_iF_i^{[2]}+\epsilon_{i,2}+{e_{i,2}}/{2})^2+{e_{i,2}^2}/{4}.$
Then, Eq.~\eqref{dot_V1} becomes
\begin{align}
\label{dot_V2}
\frac{dV}{dt}
		  =&-\sum_{i\in{\cal V}}\bigg\{\Phi_{i}\t K_iK_i\Phi_{i}+a_i^2+b_i^2\bigg\}+\sum_{i\in{\cal V}}d_i
\end{align}
with 
\begin{align}
\label{replace_val}
a_i:=&\Phi_{i}\t K_iF_i^{[1]}+\epsilon_{i,1}+\frac{e_{i,1}}{2},~ b_i:=\Phi_{i}\t K_iF_i^{[2]}+\epsilon_{i,2}+\frac{e_{i,2}}{2},\nonumber\\
d_i:=&\frac{e_{i,1}^2+e_{i,2}^2}{4}.
\end{align}
Denoting 
\begin{align}
\label{value_replace}
\Xi:=&-\sum_{i\in{\cal V}}\bigg\{\Phi_{i}\t K_iK_i\Phi_{i}+a_i^2+b_i^2\bigg\},
\end{align}
one has $\int_{0}^t\Xi(s)ds\leq 0, t\in[0, T)$,
because of $\Phi_{i}\t K_iK_i\Phi_{i}\geq0, a_i^2\geq0, b_i^2\geq0$.
From the condition~\eqref{convergence_e} and the definition of $d_i$ in \eqref{replace_val}, there exists a real numbers $\delta_1>0$ such that $\sum_{i\in{\cal V}}^nd_i(t)<\delta_1$ for $t\in[0,T]$.  
Then, it follows from Eqs.~\eqref{dot_V2}, \eqref{value_replace} that ${dV}/{dt}\leq\Xi+\delta_1$.
Using the comparison principle~\cite{khalil2002nonlinear}, one has that $V(T)\leq\int_{0}^T\Xi(s)ds+\delta_1T+V(0)\leq\delta_1T+V(0)$.

Under Assumption~\ref{assp_intial_value}, one has that $V(0)$ is bounded. 
Meanwhile, the constant $\delta_1T$ is upper bounded, so is $V(T)$. It then follows from~Eq.~\eqref{V_1} that $\sum_{i\in\mathcal{V}}\sum_{k\in\mathcal N_i(T)}\int_{\|\Omega_{i,k}(T)\|}^{R}\alpha(s)ds$ is bounded as well.
Using the fact that $\lim_{\|\Omega_{i,k}\|\rightarrow r^{+}} \alpha(\|\Omega_{i,k}\|)=\infty,  \forall i\neq k\in {\cal V}$,
with $r^{+}$ given in \eqref{alpha}, it can be concluded that $\|\Omega_{i,k}(T)\|$ stays in $\mathbb{S}$ in \eqref{unique_space} because of the continuity of $V(t)$, i.e., $\|\Omega_{i,k}(T)\|>r$. This contradicts Eq.~\eqref{err_claim}. So we conclude that there is no such a finite $T$ for \eqref{err_claim} (i.e., $T=\infty$) and  $\|\Omega_{i,k}(t)\|\in(r, \infty),  i\in\mathcal V, k\in\mathcal N_i, \forall t \geq 0$. It further guarantees that 
$\|\Omega_{i,k}(t)\|\neq r, \mbox{or}, \|\Omega_{i,k}(t)\|\neq 0, \forall i\in\mathcal V, k\in\mathcal N_i, t\geq0$, and the proof is thus completed.
\end{proof}

\subsection{ Surface Convergence and Maneuvering}

\begin{lemma}
\label{lemma_step_2}
Under Assumption~\ref{assp_derivative},  a multi-robot system~\eqref{kinetic_F} with~the CGVF~\eqref{desired_law}, \eqref{de_eta}
achieves surface convergence and maneuvering on the desired surface $\mathcal S^{phy}$, i.e., Objectives 1-2 in Definition~\ref{CPF_definition}.
\end{lemma}

\begin{proof}
\label{proof_step2}
First of all, recalling the definition of $\Xi$ in Eq.~\eqref{value_replace}, one has that
$\int_{0}^t\Xi(s)ds$ is monotone. It then follows from Eq.~\eqref{dot_V2} that $\int_{0}^t\Xi(s)ds\geq V(t)-\sum_{i\in{\cal V}}\int_{0}^td_i(s)ds-V(0), \forall t>0.$
From the fact that $\lim_{t\rightarrow\infty}d_i(t)=0$ exponentially with Eqs.~\eqref{convergence_e} and \eqref{replace_val}, it then holds that $\sum_{i\in{\cal V}}\int_{0}^td_i(s)ds$ is bounded.
Moreover, as $V(0)$ and $V(t), \forall t>0$ are all bounded in view of Lemma~\ref{lemma_bounded_P},
one has that $\int_{0}^t\Xi(s)ds$ is lower bounded, which then implies that $\int_{0}^t\Xi(s)ds$ has a finite limit as $t\rightarrow\infty$.

Bearing in mind that $V(t)$ is bounded in Lemma~\ref{lemma_bounded_P}, it follows from Eq.~\eqref{V_1} that $\Phi_i, {\Omega}_{i,k}, \widetilde{\Omega}_i, \eta_{i,1}, \eta_{i,2}$ are bounded as well. 
With the bounded first and second derivatives of $f_{i,j}(\omega_{i,1}, \omega_{i,2}), i\in{\cal V}, j\in\mathbb{Z}_1^n$ in Assumption~\ref{assp_derivative}, it follows from Eqs.~\eqref{dynamic_path_4},~\eqref{replace_val} and \eqref{value_replace} that the components of $\Phi_i, F_i^{[1]}, F_i^{[2]}, \dot{F}_i^{[1]}, \dot{F}_i^{[2]}, \eta_{i,1},  \eta_{i,2}, \dot{\eta}_{i,1},  \dot{\eta}_{i,2}, e_{i,1}, e_{i,2}, \dot{e}_{i,1}, \dot{e}_{i,2}$ in $\dot{\Xi}$ are all bounded as well, which implies that $\Xi$ is uniformly
continuous in $t$. Then, from Barbalat’s lemma~\cite{khalil2002nonlinear}, one has that $\lim_{t\rightarrow\infty}\Xi(t)=0.$
Since $a_i^2\geq0, b_i^2\geq0, \Phi_{i}\t K_iK_i\Phi_{i}\geq0, k_{i,j}>0, i\in\mathcal V, j\in\mathbb{Z}_1^n$ in Eq.~\eqref{value_replace}, 
one has $\lim_{t\rightarrow\infty} a_i(t)= 0, \lim_{t\rightarrow\infty} b_i(t)= 0, \lim_{t\rightarrow\infty} \Phi_{i}(t)=\mathbf{0},$
which implies that $\lim_{t\rightarrow\infty} \phi_{i,j}(p_i(t)) =0, i\in\mathcal V, j\in\mathbb{Z}_1^n$, i.e., surface convergence in Definition~\ref{CPF_definition} is achieved. 
Since $F_i^{[1]}, F_i^{[2]}$ are bounded under Assumption~\ref{assp_derivative}, one has that 
\begin{align}
\label{convergence_Phi}
& \lim_{t\rightarrow\infty}\Phi_i\t(t) K_iF_i^{[1]}(t)=0,~\lim_{t\rightarrow\infty}\Phi_i\t(t) K_iF_i^{[2]}(t)=0,  i\in\mathcal V.
\end{align}
Meanwhile, from the fact that $\lim_{t\rightarrow\infty}e_{i,1}(t)=0, \lim_{t\rightarrow\infty}$ $e_{i,2}(t)=0$ in \eqref{convergence_e}, it follows from Eqs.~\eqref{replace_val} and \eqref{convergence_Phi} that
\begin{align}
\label{convergence_coordinated_term}
& \lim_{t\rightarrow\infty}\epsilon_{i,1}(t)=0,~\lim_{t\rightarrow\infty}\epsilon_{i,2}(t)=0,  i\in\mathcal V.
\end{align}
Combining \eqref{convergence_e}, \eqref{convergence_Phi} and \eqref{convergence_coordinated_term} together, one has that the derivative of $\widetilde{\omega}_{i,1}, \widetilde{\omega}_{i,2}$ in \eqref{dynamic_path_4} satisfies $\lim_{t\rightarrow\infty}\dot{\widetilde{\omega}}_{i,1}(t)=0,~\lim_{t\rightarrow\infty}\dot{\widetilde{\omega}}_{i,2}(t)=0, ~i\in\mathcal V.$
Then, it holds that $\lim_{t\rightarrow\infty}\big(\dot{\widetilde{\Omega}}_i(t)-\dot{\widetilde{\Omega}}_k(t)\big)=0, \forall i\neq k\in \mathcal{V}$,
which implies that $\lim_{t\rightarrow\infty}\dot{\Omega}_i(t)=\lim_{t\rightarrow\infty}\dot{\Omega}_k(t)=\lim_{t\rightarrow\infty}$ $\dot{\Omega}^{\ast}(t)\neq\mathbf{0}$, i.e., surface maneuvering in Definition~\ref{CPF_definition} is achieved.
The proof is thus completed.
\end{proof}

\subsection{Ordering-Flexible Coordination}
\begin{lemma}
\label{lemma_step_3}
A multi-robot system composed of \eqref{kinetic_F} and~\eqref{desired_law} achieves ordering-flexible motion coordination on the desired surface $\mathcal S^{phy}$, i.e., Objective 3 in Definition~\ref{CPF_definition}.
\end{lemma}

\begin{proof}
From Eq.~\eqref{convergence_coordinated_term} in Lemma~\ref{lemma_step_2}, it follows from the definition of $\epsilon_{i,1}, \epsilon_{i,2}$ in \eqref{replace_coordination} that 
$\lim_{t\rightarrow\infty}\sum_{i=1}^N\big\{-c_i\widetilde{\omega}_{i,1}(t)+\eta_{i,1}(t)\big\}=0, \lim_{t\rightarrow\infty}\sum_{i=1}^N\big\{-c_i\widetilde{\omega}_{i,2}(t)+\eta_{i,2}(t)\big\}=0,  i\in\mathcal V$.
Using the fact that $\sum_{i=1}^N\eta_{i,1}=0, \sum_{i=1}^N\eta_{i,2}=0,$
one has that $\lim_{t\rightarrow\infty}\sum_{i=1}^N\widetilde{\Omega}(t)=\mathbf{0},$
which implies that
\begin{align*}
\lim_{t\rightarrow\infty}\bigg(\frac{1}{N}\sum_{i=1}^N\Omega_{i}(t)-\Omega_{\ast}(t)\bigg)=\lim_{t\rightarrow\infty}\frac{\sum_{i=1}^N\widetilde{\Omega}_i(t)}{N}=\mathbf{0}.
\end{align*}
Condition (a) of the ordering-flexible coordination in Definition~\ref{CPF_definition} is thus proved.
Next, we will prove Condition (b) of the ordering-flexible motion coordination in Definition~\ref{CPF_definition}, i.e., $ r<\lim_{t\rightarrow\infty}\|\Omega_{i, k}(t)\|<R,  \forall i\in{\cal V}, k\in \mathcal N_i(t).$
The first inequality is proved first. From Lemma~\ref{lemma_bounded_P}, one has that $\|\Omega_{i,k}(t)\|\in(r, \infty),  i\in\mathcal V, k\in\mathcal N_i, \forall t \geq 0$ holds, which implies the left-hand side of Condition (b) also holds, i.e.,  $\lim_{t\rightarrow\infty}\|\Omega_{i, k}(t)\|>r, i\in\mathcal V, k \in \mathcal N_i(t).$
Moreover, the second inequality $\lim_{t\rightarrow\infty}\|\Omega_{i, k}(t)\|<R, i\in\mathcal V, k \in \mathcal N_i(t)$ holds as well because of the definition of the sensing neighbors $\mathcal N_i(t)=\{k\in {{\cal V}},k\neq i \;\big|~ \| \Omega_{i,k}(t)\| <R\}$ in Eq.~\eqref{sening_neighbor}. Then, Conditions (a), (b) in Definition~\ref{CPF_definition} are both satisfied, which completes the proof.
\end{proof}

\subsection{Ordering-Flexible Multi-Robot Surface Navigation}

\begin{theorem}
\label{theorem_orderingfree}
Under Assumptions~\ref{assp_error}-\ref{assp_derivative}, a multi-robot system governed by \eqref{kinetic_F} and~\eqref{desired_law} is able to
achieve the ordering-flexible multi-robot {\it surface navigation}, i.e.,  Objectives 1-3 in Definition~\ref{CPF_definition}.
\end{theorem}

\begin{proof}
The conclusion follows directly from Lemmas \ref{lemma_bounded_P}-\ref{lemma_step_3}.
\end{proof}

\begin{figure}[!htb]
  \centering
  \includegraphics[width=\hsize]{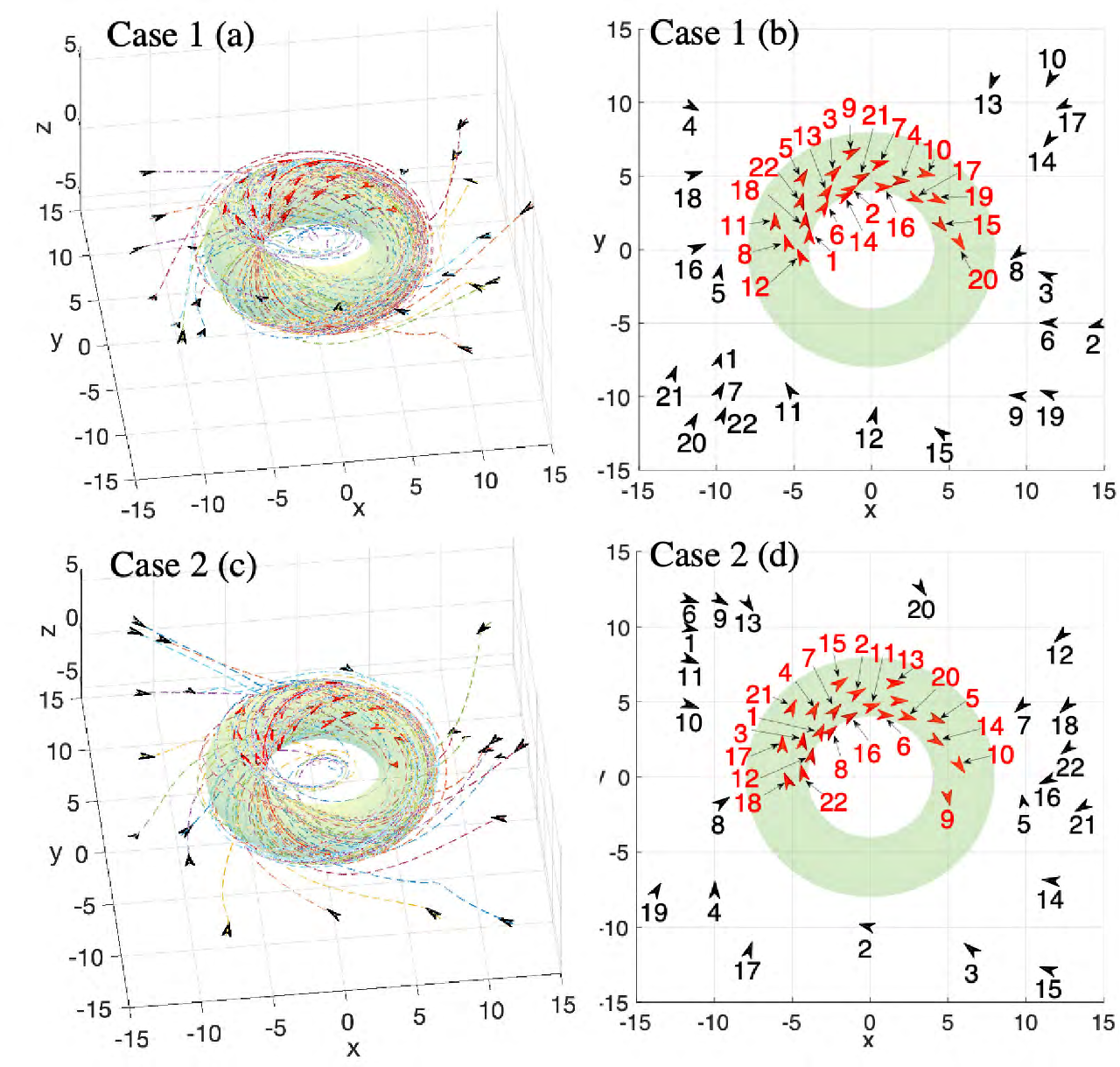}
  \caption{ (a)-(d) Two complex cases of 22-robot {\it surface navigation} tasks on a desired torus surface with different initial states. Subfigures (a), (c): Moving trajectories of the 22 robots with the CGVF~\eqref{desired_law}. Subfigures (b), (d): Top view of the initial and final positions of the robots to show the ordering-flexible coordination.
   (Here, the black and red arrows denote the initial and final positions of the robots, respectively. The dashed lines represent the trajectories of the robots. The green surface is the desired torus surface).}
  \label{20_torus_trajectories}
\end{figure}

\begin{figure}[!htb]
  \centering
  \includegraphics[width=7.6cm]{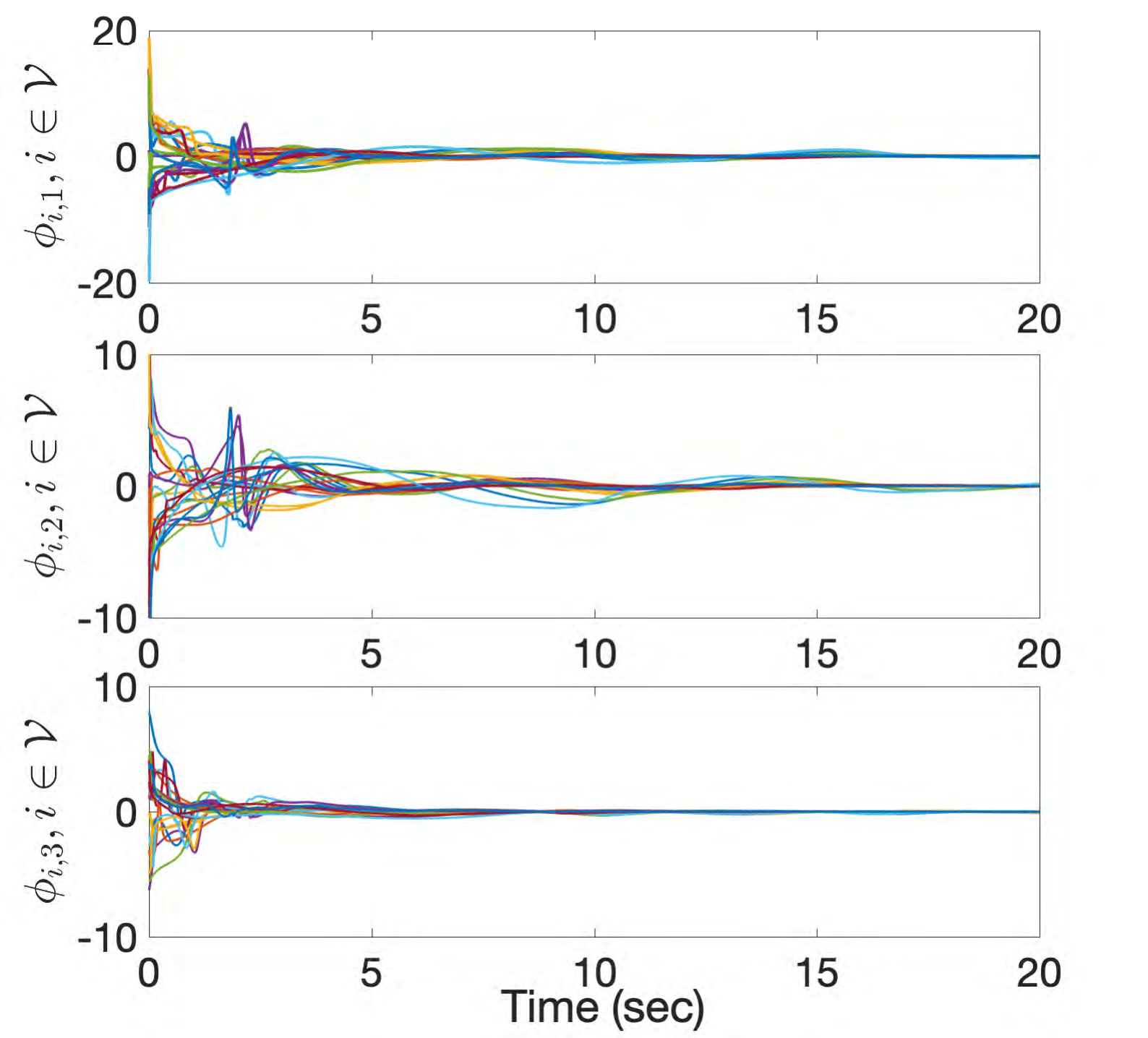}
  \caption{ Temporal evolution of the position errors $\phi_{i,1}, \phi_{i,2}, \phi_{i,3}, i\in\mathbb{Z}_1^{22}$ in Fig.~\ref{20_torus_trajectories} (Case 1).}
  \label{20_torus_trajectories_errors}
\end{figure}

\begin{figure}[!htb]
  \centering
  \includegraphics[width=7.6cm]{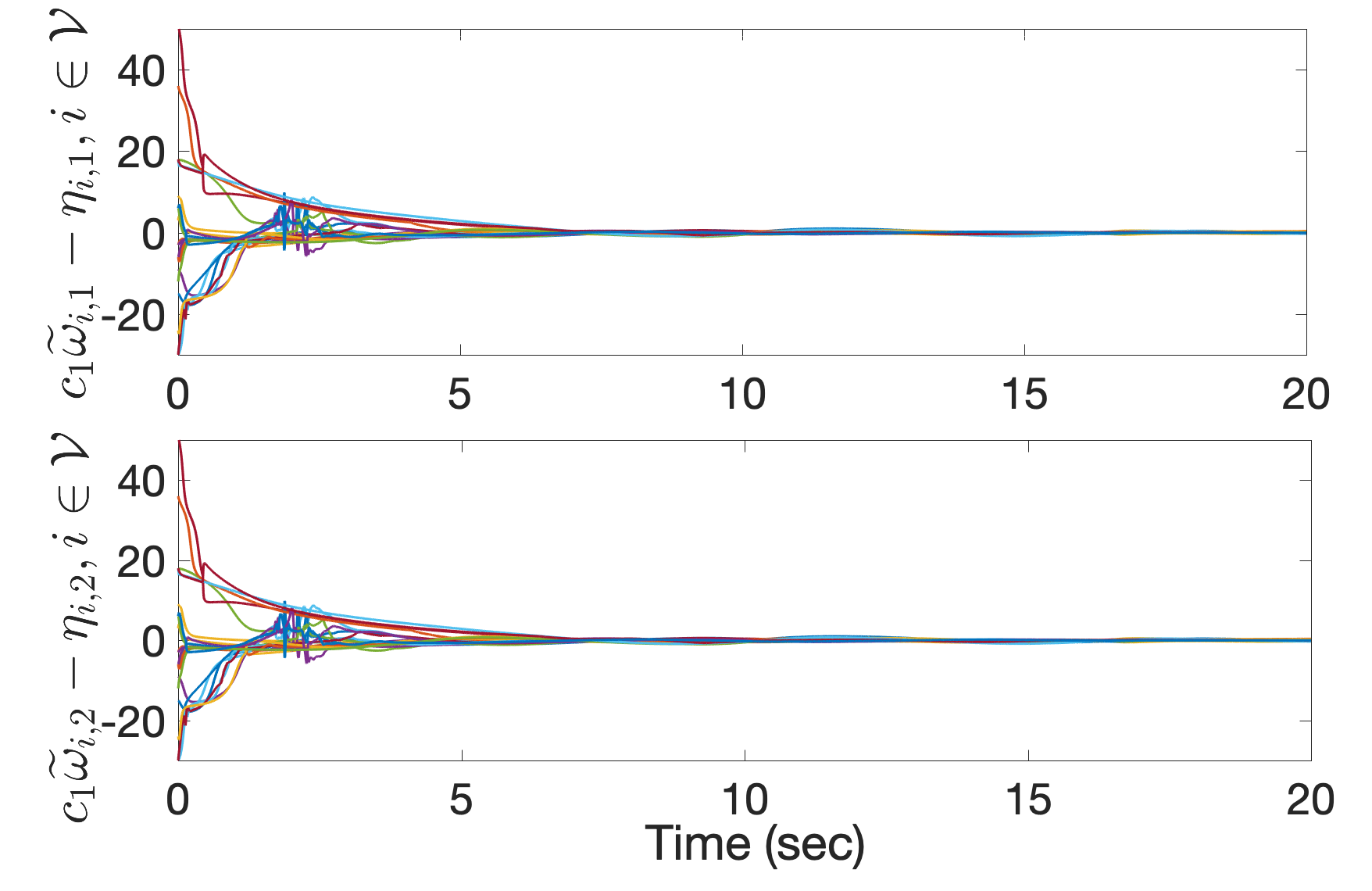}
  \caption{Temporal evolution of the ordering-flexible motion coordination terms $c_i\omega_{i,j}-\eta_{i,j}, i\in\mathbb{Z}_1^{22}, j=1, 2$ in Fig.~\ref{20_torus_trajectories} (Case 1).}
  \label{20_torus_coordinates}
\end{figure}

\section{Numerical simulation}

Consider $n=22$ robots described by Eq.~\eqref{kinetic_F}. The sensing range and safe distance are set to be $R=0.6, r=0.4$, respectively.
The potential function $\alpha(\cdot)$ is set the same as Eq.~\eqref{exam_poten_func} with the designed $R$ and $r$.
Since the auxiliary column vector $\mathbf{m}$ in Eq.~\eqref{high_eq_GVF1} is set to be $\mathbf{m}=[0,\dots,0,-1,1]\t\in\mathbb{R}^{n+2}$, one has that the dynamics of target virtual coordinates $\Omega^{\ast}$ are $\dot{\Omega}^{\ast}=[(-1)^n, (-1)^n]\t$ with Eq.~\eqref{dynamic_virtual_target}, where the initial values of the target virtual coordinates are set to be $\Omega^{\ast}(0)=[0,0]\t$ without loss of generality.
 Moreover, to estimate $\Omega^*$ in a distributed manner according to Remark~\ref{remark_c3}, we can achieve $\lim_{t\rightarrow\infty}\widehat{\Omega}_i(t)-\Omega^{\ast}(t)=\mathbf{0}, i\in \mathcal V$, exponentially.

In what follows, we consider a desired torus surface $\mathcal S_i^{phy}, i\in{\cal V}$ in the 3D Euclidean space to validate the ordering-flexible multi-robot {\it surface navigation}, of which the parametrization is $x_{i,1}=(6+2\cos\omega_{i,1})\cos\omega_{i,2}, x_{i,2}=(6+2\cos\omega_{i,1})\sin\omega_{i,2}, x_{i,3}=2\sin\omega_{i,1}, i\in{\cal V}$, satisfying Assumptions~\ref{assp_error}, \ref{assp_area} and \ref{assp_derivative}. The parameters in the CGVF~\eqref{desired_law} are set to be $k_{i,1}=0.6, k_{i,2}=0.6, k_{i,3}=0.6, c_i=3, i\in{\cal V}$. Fig.~\ref{20_torus_trajectories} illustrates two cases (i.e., Cases 1-2) of the $N=22$ robots' trajectories from different initial states (black arrows) satisfying Assumption~\ref{assp_intial_value} to the final motion coordination (red arrows) maneuvering along a desired torus surface. It is observed from the top view in subfigures (b),~(d) of Fig.~\ref{20_torus_trajectories} that the coordinations of robots are achieved with different orderings.
Additionally, to illustrate the evolution of the detailed states with different initial positions, we hereby take Case 1 in Fig.~\ref{20_torus_trajectories} (a)-(b) as an illustrative example 
in more complex 22-robot {\it surface navigation.} Fig.~\ref{20_torus_trajectories_errors} describes that $\lim_{t\rightarrow\infty}\phi_{i,1}(t)=0, \lim_{t\rightarrow\infty}\phi_{i,2}(t)=0, \lim_{t\rightarrow\infty}\phi_{i,3}(t)=0, i\in\mathbb{Z}_1^{22}$, which verifies surface approaching of all robots ${\cal V}$. Fig.~\ref{20_torus_coordinates} depicts the ordering-flexible motion coordination term approaching zeros, i.e., $\lim_{t\rightarrow\infty}\{c_i\omega_{i,j}(t)-\eta_{i,j}(t)\}=0, i\in\mathbb{Z}_1^{22}, j=1, 2$, which ensures the surface maneuvering of robots. Thereby,  the effectiveness of the proposed CGVF~\eqref{desired_law} and the feasibility of Theorem~\ref{theorem_orderingfree} are both verified in large-scale scenarios.

\section{Conclusion}
In this note, we have presented a distributed CGVF for multiple robots to realize ordering-flexible motion coordination and maneuvering on a desired 2D surface. 
The singularity-free and global-convergence properties of CGVF are both guaranteed by regarding the surface parameters as additional virtual coordinates. The ordering-flexible motion coordination is achieved by both the repulsion of neighboring robots' virtual coordinates and the attraction of the target virtual coordinates. The asymptotic convergence of multi-robot {\it surface navigation} has been rigorously analyzed by taking into account the
influence of the time-varying interaction topologies and exponentially
vanishing estimation errors. Finally, the effectiveness of the algorithm has been verified by numerical simulations.

\bibliographystyle{IEEEtran}
\bibliography{IEEEabrv,ref}

\begin{thebibliography}{10}
\providecommand{\url}[1]{#1}
\csname url@samestyle\endcsname
\providecommand{\newblock}{\relax}
\providecommand{\bibinfo}[2]{#2}
\providecommand{\BIBentrySTDinterwordspacing}{\spaceskip=0pt\relax}
\providecommand{\BIBentryALTinterwordstretchfactor}{4}
\providecommand{\BIBentryALTinterwordspacing}{\spaceskip=\fontdimen2\font plus
\BIBentryALTinterwordstretchfactor\fontdimen3\font minus
  \fontdimen4\font\relax}
\providecommand{\BIBforeignlanguage}[2]{{%
\expandafter\ifx\csname l@#1\endcsname\relax
\typeout{** WARNING: IEEEtran.bst: No hyphenation pattern has been}%
\typeout{** loaded for the language `#1'. Using the pattern for}%
\typeout{** the default language instead.}%
\else
\language=\csname l@#1\endcsname
\fi
#2}}
\providecommand{\BIBdecl}{\relax}
\BIBdecl

\bibitem{macwan2014multirobot}
A.~Macwan, J.~Vilela, G.~Nejat, and B.~Benhabib, ``A multirobot path-planning
  strategy for autonomous wilderness search and rescue,'' \emph{IEEE
  Transactions on Cybernetics}, vol.~45, no.~9, pp. 1784--1797, 2014.

\bibitem{hu2020multiple}
B.-B. Hu, H.-T. Zhang, and J.~Wang, ``Multiple-target surrounding and collision
  avoidance with second-order nonlinear multiagent systems,'' \emph{IEEE
  Transactions on Industrial Electronics}, vol.~68, no.~8, pp. 7454--7463,
  2020.

\bibitem{peng2005coordinating}
J.~Peng and S.~Akella, ``Coordinating multiple robots with kinodynamic
  constraints along specified paths,'' \emph{The International Journal of
  Robotics Research}, vol.~24, no.~4, pp. 295--310, 2005.

\bibitem{aguiar2007trajectory}
A.~P. Aguiar and J.~P. Hespanha, ``Trajectory-tracking and path-following of
  underactuated autonomous vehicles with parametric modeling uncertainty,''
  \emph{IEEE Transactions on Automatic Control}, vol.~52, no.~8, pp.
  1362--1379, 2007.

\bibitem{rysdyk2006unmanned}
R.~Rysdyk, ``Unmanned aerial vehicle path following for target observation in
  wind,'' \emph{Journal of Guidance, Control, and Dynamics}, vol.~29, no.~5,
  pp. 1092--1100, 2006.

\bibitem{kapitanyuk2017guiding}
Y.~A. Kapitanyuk, A.~V. Proskurnikov, and M.~Cao, ``A guiding vector-field
  algorithm for path-following control of nonholonomic mobile robots,''
  \emph{IEEE Transactions on Control Systems Technology}, vol.~26, no.~4, pp.
  1372--1385, 2017.

\bibitem{yao2021singularity}
W.~Yao, H.~G. de~Marina, B.~Lin, and M.~Cao, ``Singularity-free guiding vector
  field for robot navigation,'' \emph{IEEE Transactions on Robotics}, vol.~37,
  no.~4, pp. 1206--1221, 2021.

\bibitem{burger2009straight}
M.~Burger, A.~Pavlov, E.~Borhaug, and K.~Y. Pettersen, ``Straight line path
  following for formations of underactuated surface vessels under influence of
  constant ocean currents,'' in \emph{Proceedings of American Control
  Conference (ACC)}, St. Louis, MO, USA, 2009, pp. 3065--3070.

\bibitem{doosthoseini2015coordinated}
A.~Doosthoseini and C.~Nielsen, ``Coordinated path following for unicycles: A
  nested invariant sets approach,'' \emph{Automatica}, vol.~60, pp. 17--29,
  2015.

\bibitem{hu2021distributed1}
B.-B. Hu, H.-T. Zhang, B.~Liu, H.~Meng, and G.~Chen, ``Distributed surrounding
  control of multiple unmanned surface vessels with varying interconnection
  topologies,'' \emph{IEEE Transactions on Control Systems Technology},
  vol.~30, no.~1, pp. 400--407, 2022.

\bibitem{yao2019distributed}
W.~Yao, H.~Lu, Z.~Zeng, J.~Xiao, and Z.~Zheng, ``Distributed static and dynamic
  circumnavigation control with arbitrary spacings for a heterogeneous
  multi-robot system,'' \emph{Journal of Intelligent \& Robotic Systems},
  vol.~94, no.~3, pp. 883--905, 2019.

\bibitem{hu2021bearing}
B.-B. Hu and H.-T. Zhang, ``Bearing-only motional target-surrounding control
  for multiple unmanned surface vessels,'' \emph{IEEE Transactions on
  Industrial Electronics}, vol.~69, no.~4, pp. 3988--3997, 2022.

\bibitem{ghommam2010formation}
J.~Ghommam, H.~Mehrjerdi, M.~Saad, and F.~Mnif, ``Formation path following
  control of unicycle-type mobile robots,'' \emph{Robotics and Autonomous
  Systems}, vol.~58, no.~5, pp. 727--736, 2010.

\bibitem{ghabcheloo2009coordinated}
R.~Ghabcheloo, A.~P. Aguiar, A.~Pascoal, C.~Silvestre, I.~Kaminer, and
  J.~Hespanha, ``Coordinated path-following in the presence of communication
  losses and time delays,'' \emph{SIAM Journal on Control and Optimization},
  vol.~48, no.~1, pp. 234--265, 2009.

\bibitem{lan2011synthesis}
Y.~Lan, G.~Yan, and Z.~Lin, ``Synthesis of distributed control of coordinated
  path following based on hybrid approach,'' \emph{IEEE Transactions on
  Automatic Control}, vol.~56, no.~5, pp. 1170--1175, 2011.

\bibitem{sabattini2015implementation}
L.~Sabattini, C.~Secchi, M.~Cocetti, A.~Levratti, and C.~Fantuzzi,
  ``Implementation of coordinated complex dynamic behaviors in multirobot
  systems,'' \emph{IEEE Transactions on Robotics}, vol.~31, no.~4, pp.
  1018--1032, 2015.

\bibitem{de2017circular}
H.~G. De~Marina, Z.~Sun, M.~Bronz, and G.~Hattenberger, ``Circular formation
  control of fixed-wing {UAV}s with constant speeds,'' in \emph{Proceedings of
  IEEE International Conference on Intelligent Robots and Systems (IROS)},
  Vancouver, BC, Canada, 2017, pp. 5298--5303.

\bibitem{nakai2013vector}
K.~Nakai and K.~Uchiyama, ``Vector fields for {UAV} guidance using potential
  function method for formation flight,'' in \emph{AIAA Guidance, Navigation,
  and Control (GNC) Conference}, 2013, p. 4626.

\bibitem{pimenta2013decentralized}
L.~C. Pimenta, G.~A. Pereira, M.~M. Gon{\c{c}}alves, N.~Michael, M.~Turpin, and
  V.~Kumar, ``Decentralized controllers for perimeter surveillance with teams
  of aerial robots,'' \emph{Advanced Robotics}, vol.~27, no.~9, pp. 697--709,
  2013.

\bibitem{yao2021multi}
W.~Yao, H.~G. de~Marina, Z.~Sun, and M.~Cao, ``Distributed coordinated path
  following using guiding vector fields,'' in \emph{Proceedings of IEEE
  International Conference on Robotics and Automation (ICRA)}, Xi'an, China,
  2021, in press, doi: 10.1109/ICRA48506.2021.9560845.

\bibitem{hu2023spontaneous}
B.-B. Hu, H.-T. Zhang, W.~Yao, J.~Ding, and M.~Cao, ``Spontaneous-ordering
  platoon control for multirobot path navigation using guiding vector fields,''
  \emph{IEEE Transactions on Robotics}, vol.~39, no.~4, pp. 2654--2668, 2023.

\bibitem{montenbruck2017fekete}
J.~M. Montenbruck, D.~Zelazo, and F.~Allg{\"o}wer, ``Fekete points, formation
  control, and the balancing problem,'' \emph{IEEE Transactions on Automatic
  Control}, vol.~62, no.~10, pp. 5069--5081, 2017.

\bibitem{yao2022guidingArxiv}
W.~Yao, H.~G. de~Marina, Z.~Sun, and M.~Cao, ``Guiding vector fields for the
  distributed motion coordination of mobile robots,'' \emph{IEEE Transactions
  on Robotics}, vol.~39, no.~2, pp. 1119--1135, 2023.

\bibitem{chen2019cooperative}
Z.~Chen, ``A cooperative target-fencing protocol of multiple vehicles,''
  \emph{Automatica}, vol. 107, pp. 591--594, 2019.

\bibitem{zhao2013distributed}
Y.~Zhao, Z.~Duan, G.~Wen, and Y.~Zhang, ``Distributed finite-time tracking
  control for multi-agent systems: An observer-based approach,'' \emph{Systems
  \& Control Letters}, vol.~62, no.~1, pp. 22--28, 2013.

\bibitem{hong2006tracking}
Y.~Hong, J.~Hu, and L.~Gao, ``Tracking control for multi-agent consensus with
  an active leader and variable topology,'' \emph{Automatica}, vol.~42, no.~7,
  pp. 1177--1182, 2006.

\bibitem{freeman1995global}
R.~Freeman, ``Global internal stabilizability does not imply global external
  stabilizability for small sensor disturbances,'' \emph{IEEE Transactions on
  Automatic Control}, vol.~40, no.~12, pp. 2119--2122, 1995.

\bibitem{khalil2002nonlinear}
H.~K. Khalil, \emph{Nonlinear Systems}.\hskip 1em plus 0.5em minus 0.4em\relax
  Upper Saddle River, NJ, USA : Prentice Hall, 2002.

\end{thebibliography}
\end{document}